\documentclass[10pt,letterpaper,twocolumn]{article}

\usepackage[OT1]{fontenc}
\usepackage{amsmath,amssymb,amsthm}
\usepackage[margin=0.72in,columnsep=0.28in]{geometry}
\usepackage{graphicx}
\usepackage[round]{natbib}
\usepackage{booktabs}
\usepackage{multirow}
\usepackage{array}
\usepackage{algorithm}
\usepackage{algorithmic}

\newtheorem{theorem}{Theorem}
\newtheorem{corollary}{Corollary}
\newtheorem{remark}{Remark}

\makeatletter
\renewcommand\section{\@startsection{section}{1}{\z@}%
  {-3.5ex \@plus -1ex \@minus -.2ex}{2.3ex \@plus.2ex}%
  {\raggedright\normalfont\Large\bfseries}}
\renewcommand\subsection{\@startsection{subsection}{2}{\z@}%
  {-3.25ex\@plus -1ex \@minus -.2ex}{1.5ex \@plus .2ex}%
  {\raggedright\normalfont\large\bfseries}}
\makeatother

\title{\vspace{-2.2em}\textbf{Reading Less While Writing:}\\[0.2em]
       \textbf{A Closed-Form Bandwidth Dial for Streaming\\
       Multimodal Decoders}\vspace{-0.2em}}

\author{
  Yasir Mehmood\\
  \small Independent Researcher\\
  \small Lahore, Pakistan\\
  \small\texttt{yasir.mehmood@gmail.com}
  \and
  Kashif Javed\\
  \small University of Engineering and Technology (UET)\\
  \small Lahore, Pakistan\\
  \small\texttt{kashif.javed@uet.edu.pk}
}
\date{}

\begin{document}

\twocolumn[
\begin{@twocolumnfalse}
\maketitle
\begin{abstract}
\noindent
When a model translates a source---be it video or audio---into text, it
conventionally looks at the whole input before writing a single word. In offline settings this is
merely more than the task requires; in live mode, it is impossible: a caption cannot wait for
a football match to end. The streaming methods used in practice bolt on a scheduler
that reveals the input by a fixed rule (for instance, the wait-$k$ family) and always
waits for the same number of input tokens before each word, regardless of the length or pace of the input. We take a simpler route. Unlike a fixed offset, our model, ZENDAYA
(Zero-leakage Non-Decreasing Aligned \& Yoked Architecture), dynamically shapes its reading to
each input by leveraging a single knob, $\gamma$, that slides the visible window smoothly with the
sentence and in proportion to the input's own predicted length. So, while the schedule remains
deterministic, with almost nothing learned from the input except for an auxiliary length estimate, it is not entirely blind to the input the way a
constant $k$ is. Thus, turn $\gamma$ one way and the model reads everything upfront, as usual. Then, turn it the other way and it keeps pace with a live feed. This knob also fixes, in closed form, the
average fraction of the source consumed per word written, and it carries a proof: in streaming operation, no output can ever depend on input that has not yet arrived. The result is counterintuitive: seeing less can produce better text, since a flood of input
dilutes attention at the very moment the model decides what to say first. Trained
from scratch on three public benchmarks spanning two modalities (Charades-STA and
ActivityNet Captions for video and LibriHeavy for audiobook speech) our compact
29M-parameter decoder matches or outperforms the fixed schedule while reading less of the source. Its sharpest gains fall exactly where live systems operate: in our lowest-latency settings, text generation begins after a single source token, in contrast to a larger fixed wait. Even there, ZENDAYA keeps pace with the stream while holding quality, delivering streaming METEOR improvements that remain statistically significant on all three datasets. On ActivityNet Captions, it
reaches this quality with a fraction of the parameters that comparable video
captioners use, signalling that a principled schedule, not sheer size, is doing the work.
\vspace{1.4em}
\end{abstract}
\end{@twocolumnfalse}
]

\section{Introduction}\label{sec:intro}

A person interpreting a speech does not wait for the speaker to finish. They begin almost at once, a few words behind, and let their output flow alongside the incoming words. Machines that turn one sequence into another, for instance speech into a transcript or video into a caption, usually do the opposite. A standard decoder consults the \textit{entire} input before committing to its first word. For a short clip stored on disk, this costs nothing. For a live broadcast it is impossible as the caption cannot wait for the match to end before announcing the goal.

This gap between how machines translate and how the world arrives has two faces, and both trace back to the same habit of looking at everything at once.

\textbf{1. Too much input can hurt.} It is tempting to assume that more context is always better. It is not. In long documents, language models quietly neglect the middle of what they are given \cite{liu2024lost}. In video, piling on frames can lower caption quality rather than raise it, since consecutive frames make visual tokens redundant, as recently addressed by \citet{lee2025mams}. A decoder that stares at the whole source is most exposed to this dilution at the very first word, when it has the least of its own output to anchor on. And once that word is out, a left-to-right model cannot take it back.

\textbf{2. Looking at everything forecloses streaming.} If the first word depends on the whole input, the model cannot start until the whole input exists. The usual fix is to attach a separate controller that doles out the input gradually. A wait-$k$ rule, for example, reads $k$ tokens, writes one, reads another, and so on. These schedulers work, but they are rigid in a telling way. The amount they wait is a fixed constant, the same for a two-second clip and a two-minute one, blind to how long the input is or how quickly its meaning unfolds. A single setting cannot suit every input, and, as we show, no single setting suits even a single dataset once the balance between input and output shifts.

\textbf{Our fix: One knob, tuned to the data.} That balance is measurable by how many input tokens there are per output token. In some settings words race ahead of the signal, as in live transcription, where speech is written down almost as fast as it is spoken, so the model must keep pace. In others a long input funds a short summary, as in captioning a lengthy video in one sentence. Here revealing the input too eagerly is just the old look-at-everything habit in disguise. No fixed amount of waiting can be right for both settings. ZENDAYA replaces the fixed rule with a single continuous control, $\gamma$, that decides how much of the source the decoder may look at \textit{as a function of how far it has gotten}, letting the visible window grow smoothly with the sentence instead of by a rigid step. All of this comes from one mechanism, with no separate controller to train and nothing specific to audio or video baked in.

What makes this more than a convenient dial is that the same number carries three meanings at once. As a \textit{latency setting} it says how far the model trails the live input. As an \textit{interpretable budget} it fixes in closed form the average fraction of source consumed per word, so turning it up provably reads less. And as a guarantee, it ensures that, when streaming, no output can depend on input that has not yet arrived. This last one is structural, holding for trained and untrained models alike and extending to unbounded, arbitrarily-arriving streams. There is also a bonus here. Because the model cannot peek ahead, the latency measured offline is provably the latency shown live.

We put the idea to work with a deliberately small model: a 29M-parameter decoder trained from scratch, with no pretraining and no fine-tuned encoders. Across two video datasets it beats its own full-context counterpart, and on audio it matches it, all while looking at less of the source. It also meets or exceeds the fixed-schedule alternative, with the sharpest gains at the lowest latencies where the fixed rule falls apart. The scale is worth pausing on. On ActivityNet Captions, the tri-modal captioner of \citet{iashin2020mdvc} reports BLEU-4 of 1.81 and METEOR of 10.09 on ground-truth proposals with 179 million parameters. Our decoder reaches 1.97 BLEU-4 and 12.25 METEOR offline with 29 million, excluding the frozen encoders it reads from --- the figures our three-seed replication supports (Appendix~\ref{app:seeds}); the single-run sweep of Table~\ref{tab:fixed} reports 2.10 and 12.64 at a smaller $\gamma$ that the seeds do not confirm. Their evaluation averages over tIoU thresholds and adds audio and ASR text, so we read this not as a state-of-the-art claim but as evidence that a principled schedule, not sheer size, does the work.

\textbf{Contributions.} This paper contributes: \textbf{(i)} a single-parameter, progress-indexed cross-attention schedule whose endpoints are physically meaningful (i.e., offline, real-time, and deliberately starved) and whose mean source exposure obeys the closed form $\bar{E}(\gamma) \approx 1/(1{+}\gamma)$, making $\gamma$ at once a latency dial and an exposure budget, with a practical rule for setting it from a dataset's input-to-output ratio; \textbf{(ii)} a structural dependency theorem showing that, for any schedule fixed in advance and non-decreasing, the emitted stream depends only on already-seen input (trained and untrained weights alike), with a corollary extending this to unbounded streams under arbitrary asynchronous arrival; \textbf{(iii)} ZENDAYA-inf, a windowed extension for continuous streams whose only change to the decoder's training objective is a loss mask; and \textbf{(iv)} evidence across two modalities and three datasets that a compact from-scratch decoder, seeing less of the source, matches or beats a fixed-offset policy across the latency--quality frontier. We develop the finite-segment method (Sec.~\ref{sec:fix}) before extending it to unbounded streams (Sec.~\ref{sec:inf}).

\section{ZENDAYA-fix: A Tunable Horizon for a Finite Segment}\label{sec:fix}

ZENDAYA-fix operates on a fully available segment, where source features are extracted from the complete input before decoding begins and the effective output length $\hat{N}$ is predicted from that same context. It addresses context dilution but does not by itself claim end-to-end streaming. ZENDAYA-inf (Sec.~\ref{sec:inf}) removes the assumption by predicting $\hat{N}_k$ per window from already-arrived data. Every guarantee below carries over verbatim to that per-window setting.

\subsection{The Core Mechanism}

Let $X = (X_0, \dots, X_{F-1})$ be $F$ source tokens produced by a frozen modality encoder for the input segment, $\hat{N}$ an effective text length predicted per instance, and $i \in \{1, \dots, \hat{N}\}$ the current generation step. For $\gamma \geq 0$ we define the normalized progress ratio and the source horizon as
\begin{equation}
r_i = \frac{i}{\hat{N}}, \qquad \Omega_i = \left\lceil F \cdot r_i^{\,\gamma} \right\rceil. \label{eq:horizon}
\end{equation}
At step $i$, the cross-attention mask permits attention strictly to source tokens $\{0, \dots, \Omega_i - 1\}$ and blocks all indices beyond. This single structural change, i.e., a dynamically growing cross-attention window governed by one scalar, is the only modification ZENDAYA-fix makes to a standard causal decoder. The ceiling guarantees $\Omega_1 \geq 1$, so the first token always sees at least one source token and the schedule is non-decreasing in $i$ for all $\gamma \geq 0$.

\subsection{The Bandwidth Dial}

The exponent $\gamma$ has a direct physical reading: \textit{the rate at which source information becomes available relative to the rate at which text is produced}. Rather than a switch between an offline mode and a streaming mode, it traces a continuum of deployment conditions.

\begin{enumerate}
    \item \textbf{$\gamma \to 0$ (offline, full bandwidth).} For any $r_i \in (0,1]$, $r_i^{\gamma} \to 1$, so $\Omega_i \to F$ at every step: the model sees the whole segment from step one. \textit{The standard offline decoder is this family's $\gamma = 0$ endpoint, not a separate comparison point.}
    \item \textbf{$\gamma \in (0,1)$ (buffered streaming).} Since $x^{\gamma} > x$ on $(0,1)$, the feed runs ahead of real time, simulating a prefetch buffer. This is a safe default when text outpaces the source, giving early context for entity identification.
    \item \textbf{$\gamma = 1$ (real-time).} Source tokens accumulate linearly with the generation step, matching a constant arrival rate: $\Omega_i = \lceil F \cdot i / \hat{N} \rceil$.
    \item \textbf{$\gamma > 1$ (starved reveal).} Since $x^{\gamma} < x$ on $(0,1)$, the feed starts behind real time and accelerates late. Classically this models a degraded link, and as Sec.~\ref{sec:exp} shows it is also the natural range when the source floods a short target.
\end{enumerate}

\textbf{Exposure budget.} We define the mean source exposure as the average fraction of the source visible per emitted token,
\begin{equation}
\bar{E}(\gamma) \;=\; \frac{1}{\hat{N}} \sum_{i=1}^{\hat{N}} \frac{\Omega_i}{F} \;=\; \frac{1}{\hat{N}} \sum_{i=1}^{\hat{N}} \left(\frac{i}{\hat{N}}\right)^{\!\gamma} + \, O\!\left(\tfrac{1}{F}\right).
\end{equation}
The rounding contributes an $O(1/F)$ term per step; since the source-to-target ratio $\rho=F/\hat{N}$ is bounded below in our regimes, this error is absorbed. The sum is a Riemann sum of $x^{\gamma}$ on $[0,1]$, so
\begin{equation}
\bar{E}(\gamma) \;\longrightarrow\; \int_0^1 x^{\gamma}\, dx \;=\; \frac{1}{1+\gamma}. \label{eq:budget}
\end{equation}
The dial therefore has a second, budget reading: $\gamma = 1$ halves the average exposure, $\gamma = 0.3$ trims it to 77\%, $\gamma = 2$ to 33\%, and $\gamma \to 0$ recovers the full-exposure decoder.

\textbf{Choosing $\gamma$ from the data's geometry.} The useful range of the dial is set by the source-to-target ratio $\rho = F/\hat{N}$, and it does not end at $1$. When text outpaces source ($\rho < 1$, e.g., speech transcription), $\gamma \in (0,1)$ supplies a prefetch buffer at modest latency, and the schedule's first-step horizon $\Omega_1 = \lceil F(1/\hat{N})^{\gamma} \rceil$ is naturally small. When the source floods a short target ($\rho \gg 1$, e.g., captioning long video segments), even moderate $\gamma$ front-loads a large \textit{absolute} number of source tokens, and the just-in-time operating points migrate above $1$. A practical rule of thumb, however, makes this quantitative: to target a first-token budget of $B$ source tokens, set
\begin{equation}
\gamma^{*} \;\approx\; \frac{\ln(F/B)}{\ln \hat{N}}. \label{eq:gamma_rule}
\end{equation}
For our audio geometry ($F \approx 100$, $\hat{N} \approx 165$), a lean buffer of $B \approx 10$ source tokens calls for $\gamma^{*} \approx 0.45$, while a more generous $B \approx 30$ calls for $\gamma^{*} \approx 0.24$. The longer one is willing to buffer, the gentler the exponent. The experiments of Sec.~\ref{sec:exp} confirm this prediction, along with a sharp negative one. Where the source is the \textit{slower} stream ($\rho < 1$), a schedule that deliberately runs behind it ($\gamma > 1$) starves permanently. This is the deeper contrast with a fixed offset. Wait-$k$ owns one constant that buys a single point on the latency axis per trained model, with lag identical for every instance, while $\gamma$ repositions an entire per-instance, length-adaptive schedule.

\subsection{Streaming-Compatible Dependency}

We now make the dependency guarantee precise for arbitrary fixed parameters, letting the decoder consist of:
\begin{itemize}
	\item token and positional embeddings;
	\item $L$ layers, each applying causal self-attention, then cross-attention in which position $i$ may attend only to source indices $\{0, \dots, \Omega_i - 1\}$ (masked logits are set to $-\infty$ before the softmax), then position-wise residual, normalization, and feed-forward maps;
	\item a linear output projection.
\end{itemize}
We say a quantity \textit{depends only on} $X_{<k}$ if it is unchanged under arbitrary modification of $X$ on indices $\geq k$. If $\Omega_i = 0$ the attended set is empty, and we adopt the convention that the cross-attention block outputs zero at position $i$. Every statement below is preserved under this.

\begin{theorem}\label{thm:dep}
For any horizon schedule $i \mapsto \Omega_i$ that is fixed before decoding begins and non-decreasing in $i$, the output logits at step $i$ are a function only of the text prefix $y_{0:i}$ and the source prefix $X_{<\Omega_i}$.
\end{theorem}

\begin{proof}[Proof sketch]
By induction on layers. The base case is immediate: $h_i^{(0)}$ is a function of $y_{0:i}$ alone. At layer $\ell+1$, causal self-attention reads only $\{h_{i'}^{(\ell)} : i' \leq i\}$, and monotonicity gives $\Omega_{i'} \leq \Omega_i$, so every input depends only on $X_{<\Omega_i}$. In cross-attention the masked logits equal $-\infty$, so each index $j \geq \Omega_i$ receives softmax weight exactly zero and is absent from the normalizer, leaving $\sum_{j<\Omega_i} p_{ij}V_j$. The remaining maps are position-wise. Hence $\mathrm{logits}_i$ depends only on $y_{0:i}$ and $X_{<\Omega_i}$. A further induction over decoding steps extends this to the emitted stream. The full argument is in Appendix~\ref{app:thm}.
\end{proof}

\begin{remark}[Monotonicity closes the induction]\label{rem:mono}
The condition is not a smoothness preference: without it, a later position could inherit, through causal self-attention, source information gathered by an earlier position with a wider window, and the induction's self-attention step would fail. The power-law schedule satisfies it for every $\gamma \geq 0$, so the theorem covers the whole dial without per-instance verification.
\end{remark}

\begin{remark}[Scope and substance]\label{rem:scope}
The guarantee is structural, not learned: it concerns the computation graph alone and holds for a randomly initialized model as surely as for a trained one. At $\gamma = 0$, $\Omega_i \equiv F$ and the conclusion degenerates to ``the output depends only on $X_{<F}$''; for $\gamma > 0$, wherever $\Omega_i < F$ the theorem is a genuine, falsifiable independence claim. Finally, because emitted tokens cannot depend on unrevealed source, any latency statistic computed offline equals the one the system exhibits live---and the schedule is known in closed form a priori, not inferred from measured read/write decisions.
\end{remark}
In ZENDAYA-fix the schedule is itself a function of $\hat{N}(X)$, so the guarantee is conditional on that schedule. ZENDAYA-inf removes the conditioning by computing $\hat{N}_k$ from arrived data alone (Cor.~\ref{cor:asynch}).

\subsection{Auxiliary Length Head}\label{sec:lenhead}
The schedule's scale is set per instance by a predicted effective length $\hat{N}$. A lightweight head (global average pooling over the source features followed by two dense layers) is trained jointly with the main decoder, using a cross-entropy loss over discretized length classes weighted by $\lambda_{\mathrm{len}} = 0.1$. During training the schedule itself is built from the ground-truth target length $N=|Y|$ (Alg.~\ref{alg:anaya_fix}). The head is optimized only as an auxiliary task, and its prediction $\hat{N}$ is consulted solely at inference, when $N$ is unavailable. Predicting $\hat{N}$ from a pooled summary is the one place ZENDAYA-fix looks at the whole input ahead of decoding, and ZENDAYA-inf removes even this (Sec.~\ref{sec:inf}).
\begin{algorithm}[t]
	\caption{ZENDAYA-fix: training and inference}
	\label{alg:anaya_fix}
	\begin{algorithmic}[1]
		\STATE \textbf{Train}$(X, Y)$: \COMMENT{parallel, teacher-forced}
		\STATE \quad $N \leftarrow |Y|$ \COMMENT{target length sets the schedule}
		\STATE \quad $\Omega_i \leftarrow \lceil F \cdot (i/N)^{\gamma} \rceil$ \quad for $i = 1 \dots N$
		\STATE \quad $M_{\mathrm{cross}}[i, j] \leftarrow \mathbf{1}[\,j < \Omega_i\,]$ \quad for all $i, j$
		\STATE \quad $\mathrm{logits} \leftarrow \mathrm{Decoder}(Y_{\mathrm{in}}, X;\, M_{\mathrm{cross}})$
		\STATE \quad $\mathcal{L} \leftarrow \mathcal{L}_{\mathrm{text}} + \lambda_{\mathrm{len}} \cdot \mathcal{L}_{\mathrm{length}}$ \COMMENT{$\mathcal{L}_{\mathrm{length}}$: aux. length loss}
		\STATE \textbf{Infer}$(X)$: \COMMENT{sequential, $N$ predicted}
		\STATE \quad $\hat{N} \leftarrow \mathrm{LengthHead}(\mathrm{GAP}(X))$
		\STATE \quad $K_v, V_v \leftarrow W_K X,\, W_V X$
		\FOR{$i = 1 \dots \hat{N}$}
		\STATE $\Omega_i \leftarrow \lceil F \cdot (i/\hat{N})^{\gamma} \rceil$
		\STATE $y_i \leftarrow \mathrm{Decode}(y_{0:i},\, K_v, V_v;\, \text{horizon } [0, \Omega_i))$
		\STATE \textbf{if} $y_i = \texttt{[EOS]}$ \textbf{then break}
		\ENDFOR
		\STATE \quad \textbf{return} $y_1 \dots y_{\hat{N}}$
	\end{algorithmic}
\end{algorithm}

\subsection{Training: 2D Masking}
During teacher-forced training all $\hat{N}$ positions are computed in parallel. We precompute the boolean mask $M_{\mathrm{cross}} \in \mathbb{B}^{\hat{N} \times F}$, whose row $i$ is true exactly on $[0, \Omega_i)$, and pass it to every cross-attention layer (see Alg.~\ref{alg:anaya_fix}). A single forward pass thus reproduces the dependency structure of Theorem~\ref{thm:dep} across the whole batch, with no sequential loop.
\section{ZENDAYA-inf: Unbounded Streams via Context-Conditioned Chunking}\label{sec:inf}

ZENDAYA-fix operates on a finite, fully available segment. Genuinely long or unbounded streams, such as live commentary, continuous interpretation, or rolling transcription, cannot be held in memory. ZENDAYA-inf extends the bandwidth dial to a sliding source window of width $W$, carrying into each window the text $Y_{\mathrm{history}}$ generated for the one before it.

\subsection{Prefix Forcing with Loss Masking}
Training presents the model with consecutive target sentences drawn from the same continuous stream. The earlier sentence acts as historical context: its loss weight is $0$ and its source horizon is fixed at $r = 1$, i.e., fully open, since it lies entirely in the past relative to the active sentence. The progressive horizon $\Omega_{k,j}$ applies only to the active target sentence in window $k$. This prefix forcing is the only change to the decoder's training objective relative to ZENDAYA-fix. For each pair of consecutive ground-truth sentences $(S_{k-1}, S_k)$, the loss is masked to zero on $S_{k-1}$ and applied fully to $S_k$. 
Following standard practice in streaming captioning, $Y_{\mathrm{history}}$ carries the most recently completed sentence rather than the entire transcript, so the deployed condition matches the training condition exactly. 

\subsection{Per-Window Length Prediction (Strictly Causal)}
Theorem~\ref{thm:dep} guarantees that \textit{generation} at step $j$ in window $k$ depends only on the first $\Omega_{k,j}$ source tokens of $w_k$, but $\Omega_{k,j}$ is derived from $\hat{N}_k$, which must itself be predicted without future information. We predict $\hat{N}_k$ from three pooled summaries, each available before window $k$ is decoded: (i) the first $B$ source tokens of window $k$, where $B = \lceil W \cdot (1/N_{\max})^{\gamma} \rceil$ and $N_{\max}$ is the longest decode permitted within a window (Sec.~\ref{sec:exp}), so the buffer never exceeds what the first token would license; and (ii) the features of the preceding window $w_{k-1}$; (iii) the text $Y_{\mathrm{history}}$ generated for that preceding window. Every input to this prediction has already arrived by the time it is used, so no component depends on future source data.

\subsection{Dependency Under Asynchronous Arrival}
Within window $k$, decoding proceeds over steps $j = 1, \dots, \hat{N}_k$ with prescribed horizon $\Omega_{k,j} = \lceil W \cdot (j/\hat{N}_k)^{\gamma} \rceil$, so the schedule never asks for more than $\Omega_{k,j}$ tokens at step $j$ and decoding may begin once the first $\Omega_{k,1}$ have arrived. What it cannot control is whether arrival keeps pace. Let $A_k(t) \in \{0, \dots, W\}$ denote the number of tokens of window $k$ encoded and cached by wall-clock time $t$, and let $t_j$ be the wall-clock time at which decoding reaches step $j$. We define the \textit{effective horizon}
\begin{equation}
\Omega^{\mathrm{eff}}_{k,j} := \min\bigl( \Omega_{k,j},\, A_k(t_j) \bigr), \label{eq:asynch}
\end{equation}
the source prefix actually licensed at step $j$ once both the schedule and physical arrival are accounted for.

\begin{algorithm}[H]
	\caption{ZENDAYA-inf: unbounded streaming inference}
	\label{alg:anaya_inf}
	\begin{algorithmic}[1]
		\STATE $Y_{\mathrm{history}} \leftarrow \varnothing$;\quad $k \leftarrow 0$
		\WHILE{stream continues}
		\STATE buffer first $B$ tokens of window $k$
		\STATE $w_k \leftarrow X[kW : (k{+}1)W]$
		\STATE $\hat{N}_k \leftarrow \mathrm{LengthHead}(\mathrm{buffer}(w_k);\, w_{k-1};\, Y_{\mathrm{history}})$
		\STATE $K_{v,k}, V_{v,k} \leftarrow$ projections of $w_k$
		\STATE $S_k \leftarrow \varnothing$
		\FOR{$j = 1 \dots \hat{N}_k$}
		\STATE $\Omega_{k,j} \leftarrow \lceil W \cdot (j/\hat{N}_k)^{\gamma} \rceil$
		\STATE decode next token attending to $Y_{\mathrm{history}}$ and $w_k$ masked to $[0, \Omega_{k,j})$
		\STATE append token to $S_k$
		\STATE \textbf{if} token $= \texttt{[EOS]}$ \textbf{then break}
		\ENDFOR
		\STATE $Y_{\mathrm{history}} \leftarrow S_k$;\quad $k \leftarrow k + 1$
		\ENDWHILE
	\end{algorithmic}
\end{algorithm}

\begin{corollary}\label{cor:asynch}
For any non-decreasing arrival process $A_k$, the map $j \mapsto \Omega^{\mathrm{eff}}_{k,j}$ is non-decreasing, and Theorem~\ref{thm:dep} applies verbatim with $\Omega^{\mathrm{eff}}_{k,j}$ in place of $\Omega_{k,j}$: generation at step $j$ is a function only of the text so far, the pre-decoding inputs that set the schedule, and the source tokens of $w_k$ with index $< \Omega^{\mathrm{eff}}_{k,j}$. No emitted token ever depends on a source token that has not arrived by the time it is produced.
\end{corollary}

\begin{proof}
The map $j \mapsto A_k(t_j)$ is non-decreasing as a composition of non-decreasing maps, and the pointwise minimum of two non-decreasing maps is non-decreasing, so $j \mapsto \Omega^{\mathrm{eff}}_{k,j}$ is a valid horizon schedule. Being non-decreasing is the only property Theorem~\ref{thm:dep} requires. The schedule is fixed by $\hat{N}_k$, which itself is computed from the $B$-token buffer, the preceding window $w_{k-1}$, and $Y_{\mathrm{history}}$. All three are available before window-$k$ decoding begins and consist exclusively of arrived data. Applying Theorem~\ref{thm:dep} to the window-$k$ decode with schedule $\Omega^{\mathrm{eff}}_{k,j}$ yields the dependency claim, and $\Omega^{\mathrm{eff}}_{k,j} \leq A_k(t_j)$ ensures that every licensed source token has arrived by time $t_j$.
\end{proof}

The guarantee holds however the prescribed schedule relates to true arrival. When tokens arrive at least as fast as the schedule asks, the effective horizon is the prescribed one and the model restricts itself as designed. Our experiments occupy this regime, and Algorithm~\ref{alg:anaya_inf} instantiates it. When arrival falls behind, as on a degraded link, the model is provably limited to what has arrived, with no separate-case logic.
\section{Related Work}
Our contribution meets three active lines of work.

\textbf{Variable attention masking in streaming ASR.} Training an encoder under masks that expose differing amounts of future audio, so a single model covers both streaming and full-context use, is established practice in speech recognition \cite{swietojanski2023variable, moritz2021dual}. The \textit{principle} of mask-induced streaming compatibility is therefore not new, but its target is different as these methods mask a single encoder's self-attention over its own modality, indexed by \textit{input-frame position}. ZENDAYA masks \textit{cross-modal} attention from decoder to whichever source encoder is in use, indexed by \textit{generation progress}, $i/\hat{N}$.

\textbf{Simultaneous generation via alignment policies.} Simultaneous translation schedules a learned or scheduled alignment between source consumption and target emission, such as wait-$k$ prefix policies \cite{ma2019stacl}, multi-path wait-$k$ training \cite{elbayad2020waitk}, and monotonic attention that learns when to read versus write \cite{chiu2018mocha, ma2020monotonic}. Further, in cross-modal settings SimulSLT couples wait-$k$ with a learned boundary predictor \cite{yin2021simulslt}, and \citet{sun2024adaptive} drive an adaptive policy from a learned translatable-length estimate. ZENDAYA is related in spirit (the $\gamma$-schedule reads as a continuous, length-normalized generalization of a linear prefix policy, moved into the cross-attention mask) but differs in kind. Its visible prefix is a closed-form function of $(\gamma, i, \hat{N})$, needing no policy module, no boundary predictor, and no learned alignment, and Theorem~\ref{thm:dep} follows from the mask structure alone.

\textbf{Low-latency captioning and frame selection.} \citet{hori2021optimizing} train on partial inputs to mimic a full-context teacher with a learned timing detector, retaining 94\% of quality from 28\% of frames---the trade-off they optimize empirically is the one our dial exposes analytically, and their guarantee is behavioral rather than structural. \citet{zhou2024streaming} handle unbounded video via clustering-based fixed-size memory, addressing \textit{unbounded length} where ZENDAYA addresses \textit{per-step dependency structure} within a window. That excess source tokens can hurt quality is documented in captioning \cite{lee2025mams}, countered by input-adaptive selection applied once the whole input is in hand. ZENDAYA instead reduces exposure \textit{dynamically across generation}, never requires the full input up front, and yields Theorem~\ref{thm:dep} as a by-product rather than a separate guarantee to establish.

\section{Experiments}\label{sec:exp}

\begin{figure*}[!t]
\centering
\includegraphics[width=\linewidth]{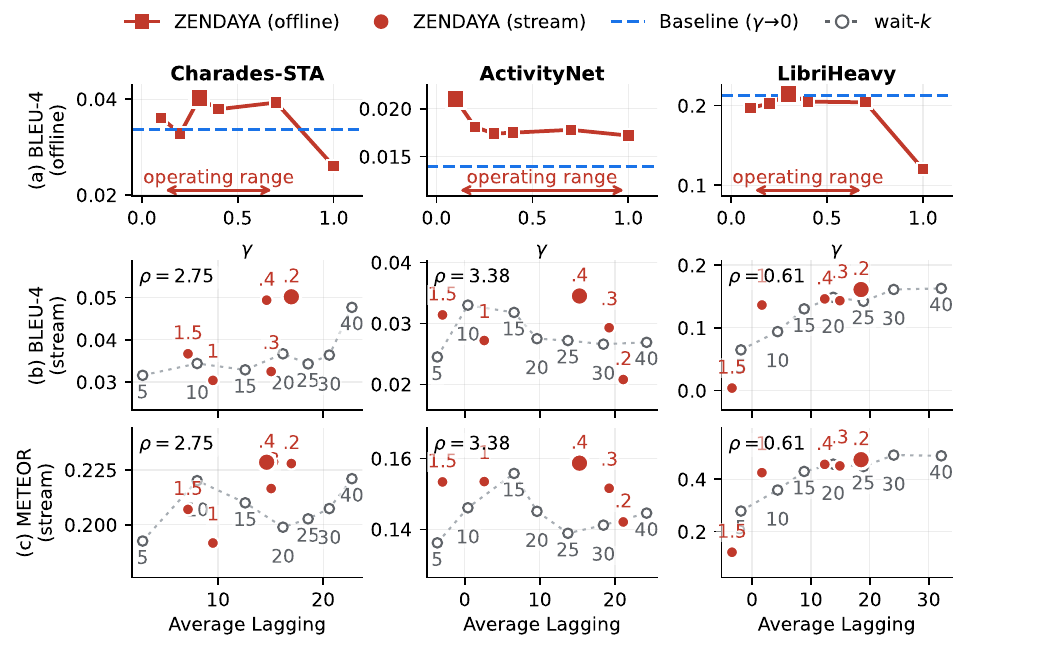}
\caption{(a) Offline BLEU-4 versus $\gamma$; blue dashed is the full-exposure baseline. (b, c) Streaming BLEU-4 and METEOR versus AL; ZENDAYA in red (labeled by $\gamma$, best enlarged), wait-$k$ in gray.}
\label{fig:main}
\end{figure*}

Our experiments answer three questions. \textbf{(Q1)} Offline: does restricting exposure cost or gain quality against the full-exposure member of the same family, at matched capacity? \textbf{(Q2)} Streaming: how does ZENDAYA's latency--quality frontier compare with wait-$k$, the canonical fixed-offset policy, in both source--target regimes? \textbf{(Q3)} Do the analytic dials---the exposure budget of Eq.~\eqref{eq:budget} and the $\gamma^{*}$ rule of Eq.~\eqref{eq:gamma_rule}---predict where each method is strong?

\subsection{Setup}

\begin{table}[t]
\centering
{\scriptsize
\setlength{\tabcolsep}{3.5pt}
\renewcommand{\arraystretch}{0.95}
\resizebox{\columnwidth}{!}{%
\begin{tabular}{l c c c}
\toprule
 & Charades-STA & ActivityNet & LibriHeavy \\
 & \cite{gao2017tall} & \cite{krishna2017anet} & \cite{kang2024libriheavy} \\
\midrule
Modality & video & video & audio \\
Frozen encoder & CLIP B/32 & C3D \texttt{fc7} & Whisper-base \\
& \cite{radford2021clip} & \cite{tran2015c3d} & \cite{radford2023whisper} \\
Token rate & 3\,fps & pre-extracted & $\approx$3\,/s \\
Feature dim. & 512 & 500 (PCA) & 512 \\
$F$ ($=W$) & 55 & 135 & 100 \\
$N$ ($=N_{\max}$) & 20 & 40 & 165 \\
$\rho = F/N$ & 2.75 & 3.38 & 0.61 \\
Regime & source-dense & source-dense & text-dense \\
\bottomrule
\end{tabular}%
}}
\caption{Datasets and configuration. $F$ and $N$ are per-corpus truncation caps and $\rho$ is their ratio; the schedule uses each instance's own lengths. One annotated segment per window, so $W = F$ and $N_{\max} = N$.}
\label{tab:setup}
\end{table}


\textbf{Models.} To attribute differences to the cross-attention schedule rather than to representation quality or pretraining, every decoder is trained \textit{from scratch}: $\sim$29M parameters, 4 layers, $d = 768$, 8 heads, feed-forward 1536, dropout 0.1, a per-dataset 8k BPE vocabulary, and AdamW \cite{loshchilov2019adamw} with cosine decay and warmup to $2\times10^{-4}$ for at most 40 epochs\footnote{Training uniformly converges in fewer than 20 epochs, batch sizes kept strictly uniform across both ZENDAYA and the baseline.}. All models are implemented in TensorFlow and trained on a single T4 GPU on either Kaggle or Google Colab. Source features come from the frozen public encoders of Table~\ref{tab:setup} and are never updated. All models decode greedily to at most $N_{\max}$ steps, stopping at [EOS]. Training is unseeded by default, so exact replication of any individual run is not expected; every headline margin instead carries a paired-bootstrap interval. The fixed-window sweep is additionally repeated under three fixed seeds in Appendix~\ref{app:seeds}, while the streaming results reported below are single runs.

\textbf{Datasets.} The three corpora span the source-to-target ratio $\rho$ on both sides of unity. Charades-STA's action segments overlap heavily in time, so we exclude overlapping consecutive segments when forming streaming $(S_{k-1}, S_k)$ pairs, keeping the causal assumption of Sec.~\ref{sec:inf} intact by construction. ActivityNet Captions supplies the most source-dense geometry we test as its long and heterogeneous segments each carry a single human-written caption. We use the publicly released C3D features so that the visual front end is fixed and independently reproducible. LibriHeavy is large enough that both families would eventually excel, so we train on a 6k-utterance subset with a book-disjoint test split, placing every model in a data-austere regime where the schedule, not scale, must do the work. Table~\ref{tab:setup} gives each corpus's encoder, geometry, and regime.

\textbf{Baselines.} Every comparison stays inside one family. The offline baseline is the $\gamma \to 0$ endpoint of the identical architecture. The wait-$k$ baselines are the identical decoder trained under a wait-$k$ cross-attention mask revealing $g(i) = k + \lceil s\,(i-1) \rceil$ source tokens at step $i$, with stride $s = \rho$ so that a full decode consumes the whole source. This length-ratio catch-up is STACL's strong variant \cite{ma2019stacl}, making the baseline the competitive form of wait-$k$ rather than a strawman. It carries no length head, since a fixed offset needs no length estimate, whereas ZENDAYA must predict the target length, from only the first tokens of a window when streaming. The auxiliary loss might also regularize the shared trunk. We disclose this asymmetry rather than adjudicate its net sign.

\textbf{Protocol and metrics.} Offline results (Table~\ref{tab:fixed}, Fig.~\ref{fig:main}a) are produced by ZENDAYA-fix. Streaming results (Table~\ref{tab:stream}, Fig.~\ref{fig:main}b,~c) are generated by ZENDAYA-inf with per-window $\hat{N}_k$ predicted causally from a $B$-token buffer (Sec.~\ref{sec:inf}). We report corpus BLEU-4 \cite{papineni2002bleu} (smoothed) and METEOR\footnote{We report both metrics because they fail differently: short single-reference captions make 4-gram overlap brittle on video, while BLEU-4 suits the long fluent audio transcripts.} \cite{banerjee2005meteor} throughout, and in the streaming case Average Lagging \cite{ma2019stacl} on the decode-step clock with $g(i) = \Omega_i$. Because the streaming schedule is causal given $\hat{N}_k$, this AL is the one the system exhibits live (Cor.~\ref{cor:asynch}). For wait-$k$, whose lag is content-blind, AL grows monotonically with $k$, though per-instance $\rho$ variation under a dataset-constant stride shifts and spreads it below $k$. Since no analytical map couples a linear $k$ to a power-law $\gamma$, we sweep both and compare frontiers directly: $\gamma \in \{0.1, 0.2, 0.3, 0.4, 0.7, 1.0\}$ offline, and $\gamma \in \{0.2, 0.3, 0.4, 1.0, 1.5\}$ against $k \in \{5, 10, 15, 20, 25, 30, 40\}$ when streaming. Anchor pairs are matched on lag where available, and best-versus-best otherwise. The three headline pairs carry a paired-bootstrap 95\% interval over 1000 resamples. The offline sweep is reported as a trend, its shape visible in Fig.~\ref{fig:main}a.

\subsection{Offline Setting: Dilution (Q1)}

Table~\ref{tab:fixed} and Figure~\ref{fig:main}a show the offline response. The best configuration is not the full-exposure endpoint on \textit{any} of the three datasets, and the margin grows with $\rho$: +19.6\% and +46.9\% on the two source-dense video datasets, where a flood of source tokens gives dilution the most room to operate. Shielding the earliest, least-anchored generation steps from the full source does not merely preserve quality but improves it. Both BLEU-4 margins clear a paired bootstrap ($p{=}0.026$ and $p{<}0.001$), and METEOR moves with them, significantly on ActivityNet. On LibriHeavy, where the audio is already text-dense, neither metric shifts significantly, exactly as that geometry predicts, and quality stays within roughly 4\% across the band $\gamma \in [0.1, 0.7]$ before collapsing at $\gamma = 1$. On the source-dense video datasets, every setting across the swept band exceeds the baseline. That breadth makes the dial practical, since $\gamma$ can be raised for latency well before quality degrades. These offline numbers come from single unseeded runs. A three-seed replication of the entire fixed-window sweep (Appendix~\ref{app:seeds}, 63 training runs) supports the geometric claim on all three corpora and sharpens its form, but it does not support the $\gamma$ selected here for ActivityNet, and under seeded means the two margins quoted above become $+34\%$ and $+20\%$, reversing their order. Appendix~\ref{app:seeds} states precisely which entries are affected and what replaces them.

\begin{table}[t]
	\centering
	{\scriptsize
		\setlength{\tabcolsep}{4pt}
		\renewcommand{\arraystretch}{0.95}
		\resizebox{\columnwidth}{!}{%
\begin{tabular}{@{}l cc cc cc@{}}
			\toprule
			& \multicolumn{2}{c}{BLEU-4} & \multicolumn{2}{c}{METEOR} & & \\
			\cmidrule(lr){2-3} \cmidrule(lr){4-5}
			Dataset & Base & Best ZENDAYA & Base & ZENDAYA & Rel. & $\bar{E}$ \\
			\midrule
			Charades    & 0.0336 & \textbf{0.0402} ($\gamma$=0.30) & 0.2096 & \textbf{0.2115} & +19.6\% & 0.77 \\
			ActivityNet & 0.0143 & \textbf{0.0210} ($\gamma$=0.10) & 0.1139 & \textbf{0.1264} & +46.9\% & 0.91 \\
			LibriHeavy  & 0.2125 & \textbf{0.2143} ($\gamma$=0.30) & 0.5627 & 0.5608 & +0.8\% & 0.77 \\
			\bottomrule
	\end{tabular}%
}}
	\caption{Fixed-window results. Best ZENDAYA is selected by BLEU-4; Rel. is that margin. Bootstrap intervals for every cell are in Appendix~\ref{app:ci}.}
	\label{tab:fixed}
\end{table}

\begin{table}[t]
	\centering
	{\scriptsize
		\setlength{\tabcolsep}{3.5pt}
		\renewcommand{\arraystretch}{0.95}
		\resizebox{\columnwidth}{!}{%
\begin{tabular}{@{}l cc ccc@{}}
			\toprule
			\textit{Charades-STA} & \textsc{z}-1.5 & \textsc{z}-0.2 & \textsc{w}-10 & \textsc{w}-20 & \textsc{w}-40 \\
			\midrule
			BLEU-4 & 0.0367 & \textbf{0.0502} & 0.0344 & 0.0367 & 0.0477 \\
			METEOR & 0.2070 & \textbf{0.2279} & 0.2201 & 0.1989 & 0.2210 \\
			AL & \textbf{7.1} & 17.0 & 8.0 & 16.2 & 22.7 \\
			\midrule
			\textit{ActivityNet} & \textsc{z}-1.5 & \textsc{z}-0.4 & \textsc{w}-5 & \textsc{w}-10 \\
			\midrule
			BLEU-4 & 0.0314 & \textbf{0.0345} & 0.0245 & 0.0330 \\
			METEOR & 0.1534 & \textbf{0.1587} & 0.1362 & 0.1461 \\
			AL & $-3.0$ & 15.3 & $\mathbf{-3.7}$ & 0.4 \\
			\midrule
			\textit{LibriHeavy} & \textsc{z}-1.0 & \textsc{z}-0.2 & \textsc{w}-10 & \textsc{w}-40 \\
			\midrule
			BLEU-4 & 0.1363 & 0.1610 & 0.0939 & \textbf{0.1628} \\
			METEOR & 0.4255 & 0.4751 & 0.3590 & \textbf{0.4899} \\
			AL & \textbf{1.7} & 18.5 & 4.3 & 32.1 \\
			\bottomrule
	\end{tabular}%
}}
	\caption{Streaming anchors. z-$\gamma$ is ZENDAYA at that $\gamma$, w-$k$ is wait-$k$. Best in row is bold.}
	\label{tab:stream}
\end{table}

\begin{table}[t]
	\centering
	{\scriptsize
		\setlength{\tabcolsep}{4pt}
		\resizebox{\columnwidth}{!}{%
\begin{tabular}{@{}ll r@{~}l r@{~}l@{}}
			\toprule
			Dataset & Anchor pair & \multicolumn{2}{c}{$\Delta$ BLEU-4} & \multicolumn{2}{c}{$\Delta$ METEOR} \\
			\midrule
			Charades-STA & z-0.2 vs w-40 & 0.0024 & [$-$.003,\,.008] & \textbf{0.0070} & [.001,\,.013] \\
			ActivityNet  & z-0.4 vs w-10 & 0.0014 & [$-$.002,\,.005] & \textbf{0.0127} & [.008,\,.018] \\
			LibriHeavy   & z-1.0 vs w-10 & \textbf{0.0426} & [.034,\,.052] & \textbf{0.0668} & [.052,\,.082] \\
			\bottomrule
	\end{tabular}%
}}
	\caption{Paired-bootstrap 95\% intervals (1000 resamples), ZENDAYA (z-$\gamma$) vs. wait-$k$ (w-$k$).  Deltas are computed on unrounded scores. Bold excludes zero.}
	
	\label{tab:ci}
\end{table}

\subsection{Streaming: One Dial, Two Regimes (Q2)}

We now turn to the streaming setting. Figure~\ref{fig:main}b,~c plots the latency--quality frontiers, Table~\ref{tab:stream} the anchor points that carry the argument, and Table~\ref{tab:ci} their bootstrap intervals.

\textbf{Charades-STA.} At $\gamma = 0.2$, ZENDAYA exceeds the strongest offset found anywhere on the wait-$k$ curve on both quality metrics while arriving \textbf{5.7 source tokens earlier}. The METEOR margin is significant, the BLEU margin positive but within noise. Against wait-20, its nearest neighbour in lag, the same setting yields \textbf{+36.8\% BLEU} and \textbf{+14.6\% METEOR}. Neither knob traces a tidy single-peaked curve, the two differing in kind rather than in tuning (Sec.~\ref{sec:fix}). Paired decodes appear in Appendix~\ref{app:qual}.

\textbf{ActivityNet Captions.} This geometry makes a sharp prediction via Eq.~\eqref{eq:gamma_rule}: with a lean first-token budget the just-in-time operating points should migrate above $1$. At comparable negative lag, where a fixed offset must starve itself, $\gamma = 1.5$ beats wait-5 by \textbf{+28\% BLEU} and \textbf{+12.6\% METEOR}. In best-versus-best, $\gamma = 0.4$ exceeds wait-10 on both metrics with METEOR advantage significant. Extending the offset to wait-40 does not recover the gap: it returns 0.0269 BLEU-4 and 0.1446 METEOR at an average lagging of 24.2, inside the existing frontier and below wait-10 on both metrics.

\textbf{LibriHeavy.} In text-dense streams the transcript outruns the signal, exactly where content-blindness costs most. At the low-lag end the fixed offset collapses while the schedule holds: $\gamma = 1.0$ retains 0.1363 BLEU at a lag of \textbf{1.7 source tokens}, 64\% of the offline optimum and 45\% above the nearest fixed offset, which needs 2.6 more tokens of lag to reach 0.0939. With METEOR up 18.5\%, both margins are decisive under the paired bootstrap (Table~\ref{tab:ci}). The regime symmetry also carries a sharp negative prediction. Since the source is here the \textit{slower} stream, a schedule that deliberately runs behind it ($\gamma = 1.5$) starves permanently and collapses to BLEU 0.0039, visible at the foot of Fig.~\ref{fig:main}b, confirming that the dial's useful range is set by $\rho$.

\subsection{Synthesis (Q3)}

Where a fixed offset must be small relative to the instance, content-blindness is expensive and ZENDAYA is clearly ahead. Where a single offset happens to suit a dataset well, ZENDAYA still beats it on METEOR at lower lag (Charades), and at none of the three anchor pairs is it significantly behind on either metric. The $\gamma^{*}$ rule predicted, before the runs, both regimes' operating points and the collapse of $\gamma>1$ where the source is slower. One mechanism, one theorem, one training recipe served $\rho = 0.61$ and $\rho = 3.38$ unmodified.

\section{Conclusion}

A single number decides how much of a stream the decoder may see while it writes. Turn it down and it reads everything upfront like a usual decoder. Then turn it up and it keeps pace with a live feed and never leans on words that have not arrived. It also fixes, in closed form, how much source each written word consumes, and its guarantee is structural, not learned. So, the latency measured offline is the latency delivered live.
The finding we did not go looking for is the one worth keeping. A compact decoder trained from scratch, shown less of its input, writes better than the same decoder shown all of it, and holds that advantage across corpora whose source and target run at opposite speeds. Seeing less tunably does better.

\textit{Future work.} We set the dial by hand, once per corpus, though the right horizon surely varies window to window: a crowded minute of speech and a static shot do not deserve the same one. Moreover, learning $\gamma$ per window is theorem-preserving, since fixing it before a window decodes keeps the schedule monotone, so a per-window learned $\gamma$ inherits every guarantee.

\textit{Limitations.} The fixed-window sweep is replicated over three seeds (Appendix~\ref{app:seeds}), and that replication revises the ActivityNet operating point of Table~\ref{tab:fixed}; the streaming results remain single runs, and the optimisation noise measured offline plausibly reaches their $\gamma$ and $k$ selections too. Learned adaptive policies remain unevaluated head-to-head. And while Corollary~\ref{cor:asynch} covers arbitrary arrival, we ran only its window-synchronized case. Thus, the asynchronous regime is proved but not measured.

\bibliography{anaya}

\clearpage
\appendix
\setcounter{table}{0}
\setcounter{figure}{0}
\renewcommand{\thetable}{A\arabic{table}}
\renewcommand{\thefigure}{A\arabic{figure}}

\section*{Appendices}
\noindent
These appendices give the full proofs of Theorem~\ref{thm:dep} and
Corollary~\ref{cor:asynch}, paired-bootstrap intervals behind every entry of
Table~\ref{tab:fixed}, eighteen streaming decode clips on Charades-STA, and a
three-seed replication of the entire fixed-window sweep. That replication
supports the central claim on all three corpora and sharpens its geometric form,
but it does not support the $\gamma$ selected for ActivityNet in
Section~\ref{sec:exp}; Appendix~\ref{app:seeds} lists precisely which entries are
affected and what replaces them. Theorem, corollary and remark numbering follows
the main text, as does notation.

\section{Full Proof of Theorem~1}
\label{app:thm}

\paragraph{Setting.} Let $M$ be a decoder built from token and positional embeddings, then $L$
identical layers, then a linear output projection. Each layer applies causal self-attention,
then cross-attention in which position $i$ may attend only to source indices
$\{0, \dots, \Omega_i - 1\}$, with masked logits set to $-\infty$ before the softmax, then
position-wise residual, normalization, and feed-forward maps. Parameters are arbitrary, trained
or untrained, and no property of the training procedure is used anywhere below. The schedule
$i \mapsto \Omega_i$ is fixed before decoding begins. In ZENDAYA-fix it is determined by
$\hat{N}(X)$, so the guarantee is conditional on that schedule; in ZENDAYA-inf the schedule is
determined by arrived data alone, which removes the conditioning.

\paragraph{Definition.} A quantity \emph{depends only on} $X_{<k}$ if, for any two source
sequences $X$ and $X'$ agreeing on indices $\{0, \dots, k-1\}$, the quantity takes the same
value, however $X$ and $X'$ differ on indices $\geq k$.

\setcounter{theorem}{0}
\begin{theorem}
For any horizon schedule $i \mapsto \Omega_i$ that is fixed before decoding begins and
non-decreasing in $i$, the output logits at step $i$ are a function only of the text prefix
$y_{0:i}$ and the source prefix $X_{<\Omega_i}$.
\end{theorem}

\begin{proof}
Write $h^{(\ell)}_i$ for the hidden state at position $i$ after layer $\ell$, with $h^{(0)}_i$
the input embedding. We prove by induction on $\ell$ that $h^{(\ell)}_i$ depends only on
$X_{<\Omega_i}$, for every position $i$ simultaneously.

\textbf{Base case.} $h^{(0)}_i$ is a function of $y_{0:i}$ alone and is independent of $X$,
hence depends only on $X_{<\Omega_i}$ vacuously.

\textbf{Inductive step.} Suppose $h^{(\ell)}_{i'}$ depends only on $X_{<\Omega_{i'}}$ for every
position $i'$. Fix a position $i$.

\textit{(i) Causal self-attention.} The output at position $i$ is a function of
$\{h^{(\ell)}_{i'} : i' \leq i\}$. For each such $i'$ the schedule is non-decreasing, so
$\Omega_{i'} \leq \Omega_i$ and therefore
$\{0, \dots, \Omega_{i'}-1\} \subseteq \{0, \dots, \Omega_i-1\}$. A quantity determined by a
subset of $X_{<\Omega_i}$ is determined by $X_{<\Omega_i}$, so every input to the block, and
hence its output $\tilde{h}_i$, depends only on $X_{<\Omega_i}$. This step is where
monotonicity is indispensable: without it a position $i$ could inherit, through an earlier
position $i'$ with $\Omega_{i'} > \Omega_i$, information about source tokens beyond its own
horizon.

\textit{(ii) Masked cross-attention.} Let $e_{ij}$ denote the pre-softmax logit for source
index $j$. By construction $e_{ij} = -\infty$ for every $j \geq \Omega_i$, so
$\exp(e_{ij}) = 0$ and those indices contribute nothing to either the numerator or the
normalizer of the softmax. The attention weights therefore satisfy $p_{ij} = 0$ for all
$j \geq \Omega_i$, and the block output reduces exactly to
\[
\sum_{j=0}^{F-1} p_{ij}\, V_j \;=\; \sum_{j < \Omega_i} p_{ij}\, V_j .
\]
Each surviving key $K_j = X_j W_K$ and value $V_j = X_j W_V$ with $j < \Omega_i$ depends only
on $X_{<\Omega_i}$, and the query is derived from $\tilde{h}_i$, which depends only on
$X_{<\Omega_i}$ by (i). Every factor entering the block output is therefore accounted for.

\textit{(iii) Position-wise maps.} The state $h^{(\ell+1)}_i$ is obtained from $\tilde{h}_i$
and the cross-attention output by residual addition, normalization, and a feed-forward map,
all deterministic and applied at position $i$ alone. It is thus a function of quantities
already shown to depend only on $X_{<\Omega_i}$.

Closing the induction at layer $L$, the state $h^{(L)}_i$ depends only on $X_{<\Omega_i}$, and
therefore so do the logits $\mathrm{logits}_i = W_{\mathrm{out}} h^{(L)}_i$.

\textbf{Extension to closed-loop decoding.} The argument above concerns teacher-forced logits,
where $y_{0:i}$ is given. Under a deterministic decoding rule, greedy selection in our
experiments, it extends to generated text by a second induction, this time over steps. The
first token is produced from a fixed start symbol, so the base case is immediate. Suppose every
emitted token $y_{i'}$ with $i' < i$ depends only on $X_{<\Omega_{i'}}$. Monotonicity gives
$X_{<\Omega_{i'}} \subseteq X_{<\Omega_i}$, so the whole prefix $y_{0:i}$ is a function of
$X_{<\Omega_i}$. Combining this with the layer induction, $\mathrm{logits}_i$ and hence $y_i$
are functions of $X_{<\Omega_i}$ alone, and the emitted stream up to step $i$ inherits the same
dependency.
\end{proof}

\paragraph{Remark on the degenerate case.} If $\Omega_i = 0$ the attended set is empty and we
adopt the convention that the cross-attention block outputs zero at position $i$, a pure
language-modeling step, under which every statement above is preserved. This case cannot arise
under the ceiling construction $\Omega_i = \lceil F \cdot (i/\hat{N})^{\gamma} \rceil$,
since $F \geq 1$ and $(i/\hat{N})^{\gamma} > 0$ give $\Omega_i \geq \Omega_1 \geq 1$ for
every admissible $F$, $\hat{N}$ and $\gamma$.\footnote{Our fixed-window implementation rounds
down rather than up, and so admits $\Omega_1 = 0$ where $F < \hat{N}^{\gamma}$. The threshold grows with $\gamma$ and so is easiest to meet at the largest exponents we sweep, as for instance on LibriHeavy at
$\gamma = 1$, whose source is generally shorter than its target. The convention above governs
those steps.}

\section{Full Proof of Corollary~1}
\label{app:cor}

Within window $k$, decoding proceeds over steps $j = 1, \dots, \hat{N}_k$ with prescribed
horizon $\Omega_{k,j} = \lceil W \cdot (j/\hat{N}_k)^{\gamma} \rceil$, non-decreasing in $j$ by
Remark~\ref{rem:mono}. The schedule never asks for more than $\Omega_{k,j}$ source tokens
at step $j$, so decoding may begin once the first $\Omega_{k,1}$ have arrived. What the
schedule cannot control is whether arrival keeps pace with it, and this section supplies the
guarantee for the case where it does not.

Let $A_k(t) \in \{0, \dots, W\}$ denote the number of source tokens of window $k$ encoded and
cached by wall-clock time $t$. It is non-decreasing in $t$, since tokens cannot un-arrive. Let
$t_j$ denote the wall-clock time at which decoding reaches step $j$; since generation proceeds
strictly forward in time, $t_j$ is non-decreasing in $j$. Define the \textbf{effective
horizon}
\begin{equation}
\Omega^{\mathrm{eff}}_{k,j} := \min\big(\Omega_{k,j},\, A_k(t_j)\big),
\end{equation}
the source prefix actually licensed at step $j$ once both the schedule and physical arrival are
accounted for.

\setcounter{corollary}{0}
\begin{corollary}
For any non-decreasing arrival process $A_k$, the map $j \mapsto \Omega^{\mathrm{eff}}_{k,j}$
is non-decreasing, and Theorem~1 applies verbatim with $\Omega^{\mathrm{eff}}_{k,j}$ in place
of $\Omega_{k,j}$. Generation at step $j$ is then a function only of the text generated so far,
the pre-decoding inputs that fix the schedule, and the source tokens of $w_k$ with index
$< \Omega^{\mathrm{eff}}_{k,j}$. In particular no emitted token depends on a source token that
has not arrived by the time that token is produced.
\end{corollary}

\begin{proof}
The map $j \mapsto A_k(t_j)$ is non-decreasing, being a composition of the non-decreasing maps
$j \mapsto t_j$ and $t \mapsto A_k(t)$. For $j \leq j'$ we then have $\Omega_{k,j} \leq
\Omega_{k,j'}$ by monotonicity of the schedule and $A_k(t_j) \leq A_k(t_{j'})$ by the above,
so $\min(\Omega_{k,j}, A_k(t_j)) \leq \min(\Omega_{k,j'}, A_k(t_{j'}))$, that is
$\Omega^{\mathrm{eff}}_{k,j} \leq \Omega^{\mathrm{eff}}_{k,j'}$. Theorem~1 further asks that the
schedule be fixed before decoding begins, and $\Omega^{\mathrm{eff}}$ is not, since it moves
with the arrival process as that process unfolds. We therefore apply the theorem pathwise: fix
any realization of $A_k$ together with the realized step times $t_j$, and along that path
$j \mapsto \Omega^{\mathrm{eff}}_{k,j}$ is a determinate non-decreasing integer sequence that no
generated token can alter, which is all that either induction of Theorem~1 uses. The dependency
claim therefore holds on every path, hence unconditionally.

It remains to check that the quantities fixing the schedule are themselves free of future
information. These are $\hat{N}_k$, and through it the initial buffer of $B$ source tokens of
window $k$, the pooled features of the preceding window $w_{k-1}$, and the text
$Y_{\mathrm{history}}$ generated for that preceding window. All three are available before
window-$k$ decoding begins, and $B \leq \Omega_{k,1}$ because $N_{\max} \geq \hat{N}_k$, so the
buffer never exceeds what the first token already licenses. The text side of the induction
concerns only $Y_{\mathrm{history}}$ and the target prefix, neither of which involves $A_k$.

Applying Theorem~1 to the window-$k$ decode with schedule $\Omega^{\mathrm{eff}}_{k,j}$ yields
the dependency claim, and $\Omega^{\mathrm{eff}}_{k,j} \leq A_k(t_j)$ ensures that every
licensed source token has arrived by time $t_j$.
\end{proof}

\paragraph{Extension across windows.} The corollary is stated for a single window. Extending it
to the unbounded stream is an induction over $k = 0, 1, 2, \dots$. The base case $k = 0$ is the
corollary with empty history. For the inductive step, the quantities fixing window $k$'s
schedule consist of the preceding window and the text generated for it, both of which the
inductive hypothesis places in the past, so the corollary applies to window $k$ unchanged.
Since each window's guarantee is stated relative to data already arrived, the conjunction over
all $k$ gives the guarantee for the stream as a whole.

\paragraph{Relation to the evaluated system.} The guarantee holds regardless of the
relationship between the prescribed schedule and true arrival. Where tokens arrive at least as
fast as the schedule requests them, $\Omega^{\mathrm{eff}}_{k,j} = \Omega_{k,j}$ and the model
self-restricts exactly as prescribed, which is the regime our experiments evaluate, precomputed
features making arrival effectively instantaneous. Where arrival falls behind, as on a degraded
link, the model is automatically and provably limited to what has actually arrived, with no
separate-case logic required. The fully asynchronous regime is therefore proved here but not
measured, as the limitations note in Section~6.

\section{Paired-Bootstrap Intervals, Fixed-Window Setting}
\label{app:ci}

Table~\ref{tab:supp-ci} gives the interval behind every entry of Table~\ref{tab:fixed}. For
each dataset we compare ZENDAYA at its selected $\gamma$ against the $\gamma \to 0$ endpoint of
the identical architecture, on the same test set in the same order, resampling clips with
replacement 1000 times. Reported values
are means over resamples and so differ from Table~\ref{tab:fixed} in the fourth decimal.

Both margins shrink monotonically as the source-to-target ratio $\rho$ falls, and neither metric
moves on the text-dense corpus: with $\rho < 1$ there is no surplus source to withhold, so
restricting exposure can neither help nor hurt.

\begin{table*}[t]
\centering
{
\small
\setlength{\tabcolsep}{5pt}
\begin{tabular}{@{}ll cc cc l@{}}
\toprule
 & & \multicolumn{2}{c}{ZENDAYA} & \multicolumn{2}{c}{Baseline ($\gamma \to 0$)} & \\
\cmidrule(lr){3-4} \cmidrule(lr){5-6}
Dataset & Metric & Score & 95\% CI & Score & 95\% CI & $\Delta$ (95\% CI) \\
\midrule
Charades-STA & BLEU-4 & 0.0401 & [0.0359, 0.0443] & 0.0337 & [0.0297, 0.0384] & \textbf{0.0064} [0.0013, 0.0113] \\
($\gamma = 0.30$) & METEOR & 0.2115 & [0.2062, 0.2166] & 0.2096 & [0.2043, 0.2151] & 0.0019 [$-$0.0041, 0.0080] \\
\addlinespace[3pt]
ActivityNet & BLEU-4 & 0.0210 & [0.0176, 0.0246] & 0.0142 & [0.0116, 0.0168] & \textbf{0.0068} [0.0028, 0.0107] \\
($\gamma = 0.10$) & METEOR & 0.1264 & [0.1215, 0.1310] & 0.1139 & [0.1095, 0.1185] & \textbf{0.0126} [0.0071, 0.0179] \\
\addlinespace[3pt]
LibriHeavy & BLEU-4 & 0.2142 & [0.2008, 0.2280] & 0.2123 & [0.1979, 0.2260] & 0.0019 [$-$0.0039, 0.0079] \\
($\gamma = 0.30$) & METEOR & 0.5608 & [0.5432, 0.5787] & 0.5627 & [0.5446, 0.5800] & $-$0.0019 [$-$0.0099, 0.0066] \\
\bottomrule
\end{tabular}
}
\caption{Fixed-window paired bootstrap, 1000 resamples. Bold marks intervals excluding zero.
Two-sided $p$ on $\Delta$ BLEU-4: 0.026 Charades-STA, $<$0.001 ActivityNet, 0.560 LibriHeavy.}
\label{tab:supp-ci}
\end{table*}

\section{Qualitative Sample}
\label{app:qual}

Table~\ref{tab:qual-supp} reports eighteen Charades-STA clips with both systems' decodes on each, referred to from
the Q2 discussion of Section~5. The protocol is symmetric: we score every decode of both systems
with sentence-level BLEU-4, take the nine clips on which each system scores highest keeping one
per distinct reference, and report both systems' decodes of each. Each half is chosen on one
system's own best terms.

\begin{table*}[t]
	\centering
	{
	\scriptsize
	\setlength{\tabcolsep}{2pt}
	\begin{tabular}{@{}>{\raggedright\arraybackslash}p{0.325\textwidth}>{\raggedright\arraybackslash}p{0.325\textwidth}>{\raggedright\arraybackslash}p{0.325\textwidth}@{}}
		\toprule
		\multicolumn{3}{@{}l@{}}{\textbf{Clips where ZENDAYA scores highest}}\\
		\midrule
		\texttt{REF: a person is sitting on a bed}\newline
		\texttt{~~z: a person is sitting on a bed}\newline
		\texttt{~~w: a person is sitting on a chair} &
		\texttt{REF: a person runs into a room}\newline
		\texttt{~~z: a person runs into a room}\newline
		\texttt{~~w: a person opens a door} &
		\texttt{REF: person drinking a glass of water}\newline
		\texttt{~~z: person drinking a glass of water}\newline
		\texttt{~~w: a person is drinking a glass of water}
		\\[3pt]
		\texttt{REF: a person is holding a bag}\newline
		\texttt{~~z: a person is holding a bag}\newline
		\texttt{~~w: a person runs into a room} &
		\texttt{REF: a person drinks a glass of water}\newline
		\texttt{~~z: person drinks a glass of water}\newline
		\texttt{~~w: person takes a drink from a glass} &
		\texttt{REF: person drinking a cup of coffee}\newline
		\texttt{~~z: person drinking a cup of coffee}\newline
		\texttt{~~w: person takes a phone}
		\\[3pt]
		\texttt{REF: person pours a cup of coffee}\newline
		\texttt{~~z: person pours a cup of coffee}\newline
		\texttt{~~w: person pours it on a table} &
		\texttt{REF: a person sits on a chair}\newline
		\texttt{~~z: person sits on a chair}\newline
		\texttt{~~w: person sits on chair} &
		\texttt{REF: a person awakens in bed}\newline
		\texttt{~~z: a person awakens in bed}\newline
		\texttt{~~w: a person is awakening in a bed}
		\\
		\midrule
		\multicolumn{3}{@{}l@{}}{\textbf{Clips where wait-$k$ scores highest}}\\
		\midrule
		\texttt{REF: person puts the bag on the table}\newline
		\texttt{~~w: person puts the bag on the table}\newline
		\texttt{~~z: person takes a bag of water} &
		\texttt{REF: person takes a drink from a glass}\newline
		\texttt{~~w: person takes a drink from a glass}\newline
		\texttt{~~z: person drinks a glass of water} &
		\texttt{REF: person drinking a cup of coffee}\newline
		\texttt{~~w: person drinking a cup of coffee}\newline
		\texttt{~~z: person drinking a glass of water}
		\\[3pt]
		\texttt{REF: a person is holding a broom}\newline
		\texttt{~~w: a person is holding a broom}\newline
		\texttt{~~z: person they open the door} &
		\texttt{REF: person drinking a glass of water}\newline
		\texttt{~~w: person drinking a glass of water}\newline
		\texttt{~~z: person drinks a glass of water} &
		\texttt{REF: person pours a cup of coffee}\newline
		\texttt{~~w: person pours a cup of coffee}\newline
		\texttt{~~z: person pours a cup of water}
		\\[3pt]
		\texttt{REF: the person pours a cup of coffee}\newline
		\texttt{~~w: person pours a cup of coffee}\newline
		\texttt{~~z: person pours a glass of water} &
		\texttt{REF: a person opens a door}\newline
		\texttt{~~w: a person opens a door}\newline
		\texttt{~~z: a person opens the door} &
		\texttt{REF: person turns on a light}\newline
		\texttt{~~w: person turns on a light}\newline
		\texttt{~~z: person turns off the light}
		\\
		\bottomrule
	\end{tabular}
}
	\caption{Streaming decodes on identical Charades-STA clips. \textsc{z}: ZENDAYA $\gamma{=}0.2$; \textsc{w}: wait-40.}
	\label{tab:qual-supp}
\end{table*}

\section{Seeded Replication of the Fixed-Window Sweep}
\label{app:seeds}
The fixed-window sweep was repeated under three seeds. Seven configurations per corpus, the $\gamma \to 0$ baseline and
six values of $\gamma$, across three corpora resulted in 63 training runs in all. A seed fixes weight
initialisation, dropout and batch order. Splits and vocabulary were already fixed upstream by
the preprocessing stage. Figure~\ref{fig:supp-seeded} and Table~\ref{tab:supp-seeded} report
the outcome.

The seeds pull the offline picture towards our geometry rather than away from it. We argue that
a larger source-to-target ratio should move the useful operating point to larger $\gamma$, and
our single runs did not show that. They put ActivityNet, the corpus of highest $\rho$, at
$\gamma = 0.10$, the smallest exponent we sweep. Under three seeds its supported settings sit at
$\gamma = 0.7$, against $\gamma = 0.3$ on Charades-STA, which is the order that $\rho = 3.38$
and $\rho = 2.75$ ask for. The negative prediction survives as well. On LibriHeavy, where
$\rho < 1$ leaves no surplus source to withhold, no $\gamma$ separates from the baseline at all,
and the collapse at $\gamma = 1$ is plain on every seed.

Where the seeds disagree with us is the $\gamma$ we chose for ActivityNet. That selection came
from a single unseeded run, which is the scope those experiments declare, and it landed at
$\gamma = 0.10$. Across three seeds the setting sits a little under the baseline on both
metrics, $0.0145$ against $0.0165$ and $0.1123$ against $0.1150$, with both gaps inside the
noise. Three of our statements lean on that run. The best-against-best entry for ActivityNet in
Table~\ref{tab:fixed} does not hold at the $\gamma$ it names. The claim that every swept setting exceeds the
baseline fails there at $\gamma = 0.1$ and $\gamma = 0.3$. And the margins we quote, $+19.6\%$
for Charades-STA against $+46.9\%$ for ActivityNet, become $+34\%$ and $+20\%$, which reverses
their order. We would replace that magnitude claim with the location claim above, which is the
sharper statement and the one the seeds support. At the corrected setting ActivityNet returns
$1.97$ BLEU-4 and $12.25$ METEOR. On Charades-STA and LibriHeavy the $\gamma$ we selected
stands.

Two things temper the picture. The useful settings are not contiguous: Charades-STA gains on
both metrics at $\gamma \in \{0.1, 0.2, 0.3\}$ and again at $0.7$, but slips just below the
baseline on METEOR at $0.4$. That is the multi-modal response over $\gamma$ our future work
anticipates, and a reason to read $\gamma$ as a region to search rather than a value to tune.
And only the fixed-window setting was repeated. The streaming results remain single runs, and
the optimisation noise we measure here plausibly reaches their $\gamma$ and $k$ selections too.
We have not measured that.

Table~\ref{tab:supp-ci} and this section are not in conflict. The bootstrap resamples test clips
within one training run and so measures sampling noise, while the seeds vary the run and so
measure optimisation noise. On ActivityNet the second turns out to be the larger of the two,
which no single-run interval could have revealed. Where they disagree, we would trust the seeded
means.

\begin{figure*}[t]
	\centering
	\includegraphics[width=\linewidth]{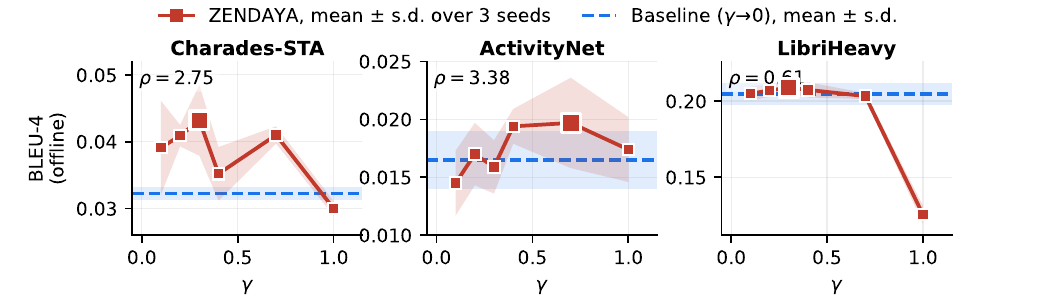}
	\caption{Offline BLEU-4 against $\gamma$, mean and standard deviation over three seeds. The
		dashed line and shaded band are the baseline mean and its own seed spread. The enlarged marker
		is the best $\gamma$ by mean on each corpus; on LibriHeavy that margin lies inside the seed
		spread, as Table~\ref{tab:supp-seeded} shows.}
	\label{fig:supp-seeded}
\end{figure*}

\begin{table*}[t]
	\centering
	\scriptsize
	\setlength{\tabcolsep}{4pt}
	\begin{tabular}{@{}l cc cc cc@{}}
		\toprule
		& \multicolumn{2}{c}{Charades-STA} & \multicolumn{2}{c}{ActivityNet} & \multicolumn{2}{c}{LibriHeavy} \\
		\cmidrule(lr){2-3}\cmidrule(lr){4-5}\cmidrule(lr){6-7}
		$\gamma$ & BLEU-4 & METEOR & BLEU-4 & METEOR & BLEU-4 & METEOR \\
		\midrule
		$\gamma\!\to\!0$ & 0.0322\,$\pm$\,0.0010 & 0.2018\,$\pm$\,0.0106 & 0.0165\,$\pm$\,0.0025 & 0.1150\,$\pm$\,0.0046 & 0.2047\,$\pm$\,0.0075 & 0.5496\,$\pm$\,0.0113 \\
		\midrule
		0.1 & 0.0391\,$\pm$\,0.0070 & 0.2160\,$\pm$\,0.0122 & 0.0145\,$\pm$\,0.0028 & 0.1123\,$\pm$\,0.0053 & 0.2049\,$\pm$\,0.0045 & 0.5525\,$\pm$\,0.0054 \\
		0.2 & \textbf{0.0409\,$\pm$\,0.0016} & 0.2168\,$\pm$\,0.0050 & 0.0170\,$\pm$\,0.0027 & 0.1156\,$\pm$\,0.0051 & 0.2070\,$\pm$\,0.0042 & 0.5533\,$\pm$\,0.0095 \\
		0.3 & \textbf{0.0432\,$\pm$\,0.0053} & 0.2191\,$\pm$\,0.0107 & 0.0159\,$\pm$\,0.0023 & 0.1155\,$\pm$\,0.0065 & 0.2088\,$\pm$\,0.0033 & 0.5522\,$\pm$\,0.0049 \\
		0.4 & 0.0352\,$\pm$\,0.0040 & 0.2000\,$\pm$\,0.0069 & 0.0194\,$\pm$\,0.0015 & \textbf{0.1210\,$\pm$\,0.0016} & 0.2074\,$\pm$\,0.0050 & 0.5476\,$\pm$\,0.0028 \\
		0.7 & \textbf{0.0410\,$\pm$\,0.0013} & 0.2098\,$\pm$\,0.0077 & \textbf{0.0197\,$\pm$\,0.0039} & \textbf{0.1225\,$\pm$\,0.0053} & 0.2033\,$\pm$\,0.0020 & 0.5422\,$\pm$\,0.0017 \\
		1.0 & 0.0300\,$\pm$\,0.0011 & 0.1911\,$\pm$\,0.0032 & 0.0174\,$\pm$\,0.0028 & 0.1193\,$\pm$\,0.0023 & 0.1255\,$\pm$\,0.0060 & 0.4117\,$\pm$\,0.0080 \\
		\bottomrule
	\end{tabular}
	\caption{Fixed-window sweep over three seeds, mean $\pm$ standard deviation. Bold marks a
		configuration that beats its baseline on every seed by a mean paired difference of at
		least three times its standard error. With three seeds these are effect sizes rather
		than significance tests, and we report them as such.}
	\label{tab:supp-seeded}
\end{table*}

\end{document}